\def\year{2017}\relax
\documentclass[letterpaper]{article} 
\usepackage{aaai17}  
\usepackage{times}  
\usepackage{helvet}  
\usepackage{courier}  
\usepackage{url}  
\usepackage{graphicx}  
\usepackage{amsfonts}
\usepackage{amssymb}
\usepackage{amsmath}
\usepackage{amsthm}
\usepackage{algorithm}
\usepackage{algorithmic}
\usepackage[colorinlistoftodos,prependcaption,textsize=tiny]{todonotes}
\newtheorem{theorem}{Theorem}
\newcommand{\namecite}[1]{\citeauthor{#1}~(\citeyear{#1})}

\begin{document}
%
\title{Summable Reparameterizations of Wasserstein Critics\\  in the One-Dimensional Setting}
\author{Christopher Grimm$^1$, \textbf{Yuhang Song}$^2$, \textbf{Michael L. Littman}$^3$ \\
  $^{1}$\textbf{University of Michigan}, $^{2}$\textbf{Beihang University},  $^{3}$\textbf{Brown University}\\
  crgrimm@umich.edu, yuhangsong@buaa.edu.cn, michael\_littman@brown.edu}
\maketitle
\begin{abstract}
Generative adversarial networks (GANs) are an exciting alternative to algorithms for solving density estimation problems---using data to assess how likely samples are to be drawn from the same distribution. Instead of explicitly computing these probabilities, GANs learn a generator that can match the given probabilistic source. This paper looks particularly at this matching capability in the context of problems with one-dimensional outputs. We identify a class of function decompositions with properties that make them well suited to the critic role in a leading approach to GANs known as Wasserstein GANs. We show that Taylor and Fourier series decompositions belong to our class, provide examples of these critics outperforming standard GAN approaches, and suggest how they can be scaled to higher dimensional problems in the future.
\end{abstract}

\section{Introduction}
Generative Adversarial Networks (GANs), introduced by \namecite{goodfellow2014generative}, have quickly become a leading technique for learning to generate data points matching samples from a distribution. GANs produce samples without directly modeling the target probability distribution. They do so by jointly training two neural networks: a \emph{generator}, which attempts to produce synthetic data points in a way that is consistent with the source distribution, and a \emph{discriminator}, which seeks to determine whether any given data point was drawn from the source distribution or the generated one. 

This joint training procedure is difficult to stabilize and many conceptual variants of the GAN framework have been proposed to improve results. We focus specifically on one such variant: the Wasserstein GAN or wGAN~\cite{Arjovsky2017WassersteinG}. While the standard GAN framework is derived as a minimax game between two agents, the wGAN framework reformulates the problem in terms of minimizing a distance metric between two probability distributions. Particularly, wGAN is formulated using the dual form of the Earth-Mover's distance, which can be reasonably approximated by a neural network. This construction results in a similar two-network setup, with one network acting as a generator and another acting as a \emph{critic}---its role is to maintain an estimate of the Earth-Mover's distance between the generator's distribution and the target distribution in a functional form that can be used as a guide to improving the generator.

Informally, Earth-Mover's distance between two probability distributions can be thought of as the amount of ``work'' that would go into transporting probability mass within each distribution to make them indistinguishable. A particularly nice property of Earth-Mover's distance is that, under mild constraints, it has a defined gradient almost-everywhere \cite{Arjovsky2017WassersteinG}, making it ideal for gradient-based optimization. To optimize over the space of critics, the optimizer must ensure that the functions it produces are $k$-Lipschitz---that the norm of their gradients is less than some scalar $k$ over the domain. A popular approach toward enforcing this constraint~\cite{Gulrajani2017ImprovedTO} involves assigning a penalty to functions that violate it on a subset of the domain. While this approach has been shown to produce visually appealing results on a variety of popular image benchmarks \cite{Gulrajani2017ImprovedTO}, there is no guarantee that the critic network will converge to the optimal critic. This failure of the critic to achieve optimality can result in generators that diverge or suffer from mode-collapse. In this work, we introduce a reparameterization of the critic network in the one-dimensional setting that has guarantees on its convergence. We show that this reparameterized critic performs better than standard gradient-penalty wGAN approaches on a set of one-dimensional simulated domains.

\section{Background}

This section provides necessary mathematical background and also summarizes related work.

\subsection{Generative Adversarial Networks}

Generative Adversarial Networks are traditionally introduced in a setting where there is some target (``real'') data source $P_r$ from which samples can be drawn. The GAN itself is defined in terms of two distinct network components: a generator, $G_\theta : \mathbb{R}^m \mapsto \mathbb{R}^n$ and a discriminator, $D_\phi : \mathbb{R}^n \mapsto [0, 1]$. The generator takes randomly sampled noise $z \sim \mathbb{P_Z}$ as input, and produces ``generated samples'' distributed according to $P_{G_\theta}$ as output. The discriminator takes real or generated data points as input and returns a scalar indicating whether a given input is real or generated. These networks are trained together in a mini-max game with the following objective:
\begin{eqnarray}\label{eq:gan_minimax}
\lefteqn{\min_{\theta}\max_{\phi} V(D_\phi, G_\theta)}\\
&=& \mathbb{E}_{P_R}\left[\log D_\phi(X)\right] + \mathbb{E}_{P_{G_\theta}}\left[ \log(1 - D_\phi(X))\right]. \nonumber
\end{eqnarray}

To optimize this objective, the generator and discriminator networks take turns, updating their own parameters while the other network's parameters are held fixed. Collectively, this objective can be thought of as the ``certainty'' of the discriminator. The generator aims to minimize this certainty, and in doing so, produce generated samples that are distributionally indistinguishable from those drawn from the real data-generating source. Conversely, the discriminator aims to maximize its own certainty by learning to discern real samples from generated ones, providing pressure on the generator to more closely match the real distribution.

\subsection{Wasserstein Generative Adversarial Networks}

As an alternative to this game-theoretic approach, Wasserstein GANs seek to minimize the Earth-Mover's distance
between two probability distributions
\begin{equation}\label{eq:earth-movers}
W(P_r, P_g) = \underset{\gamma \in \Gamma(P_r, P_g)}{\text{inf}} \mathbb{E}_{(x, y) \sim \gamma}\left[ \| x - y \| \right],
\end{equation}
where $\Gamma(P_r, P_g)$ represents the set of all joint probability distributions with $P_r$ and $P_g$ as marginal distributions. 

While the represention of Earth-Mover's distance provided in Eq.~(\ref{eq:earth-movers}) is not tractable to compute, it can be approximated in its dual form \cite{villani2008optimal}
\begin{equation}\label{eq:earth-movers-dual}
W(P_r, P_g) = \underset{\|f\|_L \leq 1}{\text{sup}} \mathbb{E}_{P_r} \left[ f(X) \right] - \mathbb{E}_{P_g} \left[ f(X) \right],
\end{equation}
where $\|\cdot\|_L \leq 1$ denotes the space of $1$-Lipschitz functions. 

Eq.~(\ref{eq:earth-movers-dual}) can be optimized similarly to the GAN setup described above. The same form of the generator network $G_\theta$ is used to produce generated data samples. However, in place of a discriminator, a critic network $f_\phi$ is used to represent the function $f$ in Eq.~(\ref{eq:earth-movers-dual}). Collectively, the resulting optimization procedure takes the following form:
\begin{equation}\label{eq:wgan-loss}
\min_{\theta} W(P_r, P_g) = \max_{\phi}  \mathbb{E}_{P_r} \left[ f_\phi(X) \right] - \mathbb{E}_{P_{G_\theta}} \left[ f_\phi(X) \right],
\end{equation}
where the critic network, $f_\phi$, is required to span a sufficiently large class of functions to approximate the supremum in Eq.~(\ref{eq:earth-movers-dual}). 

In this setting, the critic updates its parameters successively while the generator network is held fixed to maximize the above expression. After $W(P_r, P_g)$ is sufficiently maximized, the critic's parameters are frozen and the generator network takes a step to minimize $W$.

It is important to note that during this procedure special care must be taken to ensure that the critic function, $f_\phi$, belongs to the class of $1$-Lipschitz functions. The most successful method of ensuring this property is by applying a gradient penalty as an additional term in the loss function~\cite{Gulrajani2017ImprovedTO}. 

Bearing this constraint in mind, the success of gradient-based optimization relies heavily on the parameter space of the optimized function being ``nice'' in a topological sense. Particularly, a highly concave or disconnected parameter space provides a much harder optimization problem and increases the likelihood of the optimizer settling on a local optimum. Imposing a $1$-Lipschitz constraint on the optimization procedure certainly complicates the topology of the parameter space. 

\subsection{Taylor Series Approximations}

The Taylor series is a popular method that approximates a function with a sum of polynomials of increasing degree. In the one-dimensional setting, it can be expressed as 
\begin{equation}
f(x) = \sum_{n=0}^\infty \frac{f^{(n)}(x_0)}{n!}(x - x_0)^n,
\end{equation}
where $f^{(n)}$ denotes the $n$th derivative of $f$, $n!$ denotes the factorial of $n$, and $x_0$ is an arbitrary point in the domain of $f$. It is important to note the approximation is centered at $x_0$---as $x$ moves away from this central point, the approximation can become less accurate.

\subsection{Fourier Series Approximations}

The Fourier series is another method of approximating functions with a sum of functions, this time sinusoidal functions of decreasing periodicity. In the one-dimensional setting, it can be expressed as
\begin{equation}
f(x) = a_0 / 2 + \sum_{n=1}^\infty a_n \sin\left(\frac{2\pi n x}{P}\right) + b_n \cos \left(\frac{2\pi n x}{P}\right),
\end{equation}
where the sequences $(a_n)$ and $(b_n)$ are parameters particular to the function $f$ and $P$ is the period of the resulting approximation. 

\section{Derivation of the Summable Critic}

We now define a set of properties for a critic representation that we will show leads to improvements on the wGAN framework.

Let $\xi_{n,m} : [-1, 1] \mapsto \mathbb{R}$ be some sequence of functions such that both the functions and their derivatives are bounded: $\max_{x} |\xi_{n,m}(x)| \leq 1$ and 
$\max_{x} |\xi_{n,m}^\prime(x)| \leq b_{n,m}$ for all $n, m$, where $b_{n,m}$ is some bounding constant that does not depend on $x$.

We then define the weighted sum of these functions as follows:
\begin{equation}
    f_{A, \xi}(x) = \sum_{n=1}^\infty \sum_{m=1}^M A_{n,m} \xi_{n,m}(x),
\end{equation}
where $A_{n,m} \in \mathbb{R}$ and $M \in \mathbb{Z}^+$. 

From the above properties, we can derive an upper bound on the gradient of $f_A$:
\begin{equation}
\begin{aligned}
\max_x \left|f^\prime_{A,\xi}(x)\right|
&= \max_x \left|\sum_{n=1}^\infty \sum_{m=1}^M A_{n,m} \xi^\prime_{n,m}(x) \right|\\
&\leq \max_x \sum_{n=1}^\infty \sum_{m=1}^M \left| A_{n,m} \right| \left| \xi^\prime_{n,m}(x) \right| \\
&\leq \sum_{n=1}^\infty \sum_{m=1}^M \left|A_{n,m}\right| b_{n,m}.
\end{aligned}
\end{equation}

To simply our notation, we refer to this upper bound as 
\begin{equation}\label{eq:upperbound}
    \mathcal{L}_\xi(A) = \sum_{n=1}^\infty \sum_{m=1}^M |A_{n,m}| b_{n,m}.
\end{equation}

If we further assume that most functions $f: [-1, 1] \mapsto \mathbb{R}$ can be expressed by $f_{A, \phi}$ for some sequence of coefficients $(A_{n,m})$, then we can express the dual-form of the Wasserstein distance between distributions with support $[-1, 1]$ as follows:
\begin{eqnarray}\label{eq:summable_parameterization}
\lefteqn{W(P_r, P_g)}\\
&\approx& \underset{||f_{A, \xi}||_L \leq 1}{\sup}\; \mathbb{E}_{P_r}\left[ f_{A, \xi}(x)\right] - \mathbb{E}_{P_g}\left[ f_{A, \xi}(x)\right] \nonumber\\
&=& \underset{||f_{A, \xi}||_L \leq 1}{\sup}\; \sum_{n=0}^\infty \sum_{m=1}^M A_{n,m}(\mathbb{E}_{P_r}[\xi_{n,m}(x)] \\
&&- \mathbb{E}_{P_g}[\xi_{n,m}(x)]) \nonumber\\
&\approx& \underset{\mathcal{L}_\xi(A) \leq 1}{\sup}\; \sum_{n=0}^\infty \sum_{m=1}^M A_{n,m} (\mathbb{E}_{P_r}[\xi_{n,m}(x)] \\
&&- \mathbb{E}_{P_g}[\xi_{n,m}(x)]),\nonumber
\end{eqnarray}
where the last approximation follows from the property that $|f^\prime_{A, \xi}(x) | \leq \mathcal{L}_\xi(A)$ for all $x \in [-1, 1]$. 

We emphasize again how, under this new parameterization, the critic's parameters are the coefficients $(A_{n,m})$. The structure of these parameters in the constraint and objective function give us useful properties.
Specifically, our approximation of the $1$-Lipschitz constraint is convex and the optimization objective with respect to the critic is linear. Hence, we have the following theorem.

\begin{theorem} \label{thm:global_max}
\normalfont
During the optimization of a summable critic network, any setting of parameters that is a local maximum is also a global maximum. 
\end{theorem}

\begin{proof}
\normalfont
Consider two critic networks: 
\begin{equation*}
\begin{aligned}
&f_{A, \xi}(x) = \sum_{n=1}^\infty \sum_{m=1}^M A_{n,m}\xi_{n,m}(x) \\ 
&f_{B, \xi}(x) = \sum_{n=1}^\infty \sum_{m=1}^M B_{n,m}\xi_{n,m}(x),
\end{aligned}
\end{equation*}
and suppose that each critic satisfies the constraint given by Eq.~(\ref{eq:upperbound}): $\mathcal{L}_\xi(A) \leq 1$ and $\mathcal{L}_\xi(B) \leq 1$. 

Next, consider any critic produced by linearly interpolating between the parameters of $f_{A, \xi}$ and $f_{B, \xi}$. Denote such a critic as 
\begin{equation*}
f_{\alpha A + (1 - \alpha) B, \xi}(x) = \sum_{n=1}^\infty \sum_{m=1}^M \left(\alpha A_{n,m} + (1 - \alpha) B_{n,m}\right)\xi_{n,m}(x).
\end{equation*}

We can then bound the derivative of the interpolated critic by
\begin{equation*}
\begin{aligned}
&\left|\sum_{n=1}^\infty \sum_{m=1}^M \left(\alpha A_{n,m} + (1 - \alpha) B_{n,m}\right)\xi^\prime_{n,m}(x) \right| \\
&\leq \sum_{n=1}^\infty \sum_{m=1}^M \left|\alpha A_{n,m} + (1 - \alpha) B_{n,m}\right| \left|\xi^\prime_{n,m}(x)\right| \\
&\leq \alpha \mathcal{L}_\xi(A) + (1 - \alpha) \mathcal{L}_\xi(B) \\
&\leq \alpha + (1-\alpha) \\
&= 1,
\end{aligned}
\end{equation*}
where the last inequality is due to both $f_{A, \xi}$ and $f_{B, \xi}$ satisfying our modified 1-Lipschitz constraint: $\mathcal{L}_\xi(\cdot) \leq 1$. 

Hence, any critic linearly interpolated between two critics satisfying $\mathcal{L}_\xi(\cdot) \leq 1$ also satisfies the constraint. Thus, the space of critics that satisfies our constraint is convex under our summable parameterization.

Next, we observe that the Wasserstein distance under a summable critic parameterization is linear in the parameters of the critic. Collectively, the procedure of maximizing the Wasserstein distance with respect to the critic parameters now has a convex objective and a convex constraint. 

Let 
\begin{equation}
\begin{aligned}
J(A) &= W(P_r, P_g, f_{A, \xi})\\
&= \mathbb{E}_{P_r}\left[f_{A, \xi}(x)\right] - \mathbb{E}_{P_g}\left[f_{A, \xi}(x)\right]
\end{aligned}
\end{equation}
be the critic's objective function. Notice how, under the approximation in Eq.~(\ref{eq:summable_parameterization}), $J$ is the equation for a hyperplane. Thus, we have
\begin{equation*}
J(\alpha A + (1 - \alpha) B) = \alpha J(A) + (1 - \alpha) J(B)
\end{equation*}
for any settings $A$ and $B$ of the critic's parameters.

Next, suppose that $A^*$ and $B^*$ satisfy our constraint and are local maxima of $J$ with $J(B^*) < J(A^*)$ without loss of generality. If $A^*$ and $B^*$ exist, then, from the convexity of the constraint, all interpolations $Z(\alpha) = \alpha A^* + (1-\alpha)B^*$ also satisfy the constraint. Thus, we can consider the sequence of parameter settings given by $Z_t = Z(1/t)$. Clearly, $\lim_{t \to \infty} Z_t = B^*$ and each $Z_t$ satisfies the constraint. Moreover, we can write
\begin{equation*}
\begin{aligned}
J(Z_t) &= \frac{1}{t} J(A^*) + (1 - \frac{1}{t}) J(B^*) \\
&> \frac{1}{t} J(B^*) + (1 - \frac{1}{t}) J(B^*) \\
&= J(B^*).
\end{aligned}
\end{equation*}
Since we can construct a sequence of parameters $Z_t$ that approaches $B^*$ with $J(Z_t) > J(B^*)$ for each $Z_t$, this contradicts $B^*$ being a local maximum. Hence, any local maximum must be a global maximum.
\end{proof}

As a result, we can frame the optimization of the critic as a convex optimization problem, where all local maxima are global maxima. 

For appropriate settings of $\xi_{n,m}$, we can use the Fourier or Taylor bases giving
\begin{eqnarray}
\lefteqn{W_T(P_r, P_g)}\\
&=& \underset{\sum_{n=1}^\infty |n A_{n,1}| \leq 1}{\sup} \sum_{n=1}^\infty A_{n,1}\left( \mathbb{E}_{P_r}\left[x^n\right] - \mathbb{E}_{P_g}\left[x^n\right]\right) \nonumber
\end{eqnarray}
and 
\begin{eqnarray}
\lefteqn{W_F(P_r, P_g)}\\
&=& \underset{\sum_{n=1}^\infty |\pi n A_{n,1}| + |\pi n A_{n,2}| \leq 1}{\sup} \sum_{n=0}^\infty A_{n,1}( \mathbb{E}_{P_r}\left[\cos(n \pi x)\right]\nonumber\\
&&\;\;\;\;\;- \mathbb{E}_{P_g}\left[\cos(n \pi x) \right]) + \nonumber\\
&&\;\;A_{n,2} \left(\mathbb{E}_{P_r}\left[\sin(n \pi x)\right] - \mathbb{E}_{P_g}\left[\sin(n \pi x )\right] \right) \nonumber,
\end{eqnarray}
respectively. 

Thus, by representing the class of critic functions as Taylor or Fourier expansions, we obtain a clean way to enforce the $1$-Lipschitz constraint over the entire domain, while ensuring that gradient-based optimization schemes can find the globally optimal critic. We note that by enforcing an upper bound on the $1$-Lipschitz constraint, we are optimizing over a smaller set of functions. However, we have not noticed this additional constriction to affect performance empirically. 

When minimizing the expresions above, slight modifications must be made for computational tractability. First, we must choose some $N < \infty$ and cut off the remaining terms in the outer sum. Second, we must enforce our constraint as a penalty term in the loss function. Fortunately, neither of these practical considerations change the theoretical guarantees proved above. Particularly, limiting the number of terms in the outer sum to $N$ does not affect the our convexity-based arguments and embedding the constraint into the loss function as a penalty still results in a convex optimization problem.

\section{Connection with Moment Matching and Maximum Mean Discrepancy}

In this section, we review two methods in statistics that exhibit similar characteristics 
to the Wasserstein distance metric and the summable parameterization we presented in this work. 

\subsection{Maximum Mean Discrepancy}

Recent work by \namecite{Gretton2012AKT} explores the ``Maximum Mean Discrepancy'' technique to distinguish between samples drawn from different data sources. The Maximum Mean Discrepancy (also known as an Integral Probability Metric) between two data sources is defined as 
\begin{equation}
\text{MMD}(\mathcal{F}, P_X, P_Y) = \sup_{f \in \mathcal{F}} \mathbb{E}_{P_X}\left[f(X)\right] - \mathbb{E}_{P_Y}\left[f(Y)\right],
\end{equation}
where $P_X$ and $P_Y$ are the distributions of the data sources and $\mathcal{F}$ is some function class that is sufficiently rich that $P_X = P_Y$ when $\text{MMD}(\mathcal{F}, P_X, P_Y) = 0$. 

Notice how the dual-form of Wasserstein distance is a specical case of the above Integral Probability Metric when $\mathcal{F}$ is the set of $1$-Lipschitz functions. 

In their work, \namecite{Gretton2012AKT} explore using Reproducing Kernel Hilbert Spaces~\cite{berlinet2011reproducing} as the function class to perform their maximum mean discrepancy tests. This kernel-based approach is adopted by \namecite{Li2015GenerativeMM} in their work on Generative Moment Matching Networks. This work offers a method that competes directly with GANs. Rather than a mini-max game between a generator and discriminator, Generative Moment Matching Networks boast only needing a generator network that is trained to minimize the Maximum Mean Discrepancy between the real and generated sources. The authors note that their use of kernels approximates matching the moments of the sampled and generated random variables. 

\subsection{Moment Matching}

Moment matching, also known as the ``method-of-moments'' is the process of fitting a model to a distribution by sampling from that distribution and setting the model's parameters to be the distribution's sampled moments. In general, moments can refer to any set of functions that characterize the behavior of a random variable, but they are most commonly represented as the random variable raised to different powers. For any $n \geq 1$, we denote the $n$th moment of a random variable $X$ as 
\begin{equation}
m_n(X) = \mathbb{E}_{P_x} \left[ X^n\right].
\end{equation}

Particularly notice that for critics represented by the Taylor series parameterization, the Wasserstein distance can be expressed as a sum of weighted moments:
\begin{equation}
W_T(P_r, P_g) = \max_{\sum_{n=1}^\infty |n A_{n,1}| \leq 1} \sum_{n=1}^\infty A_{n,1}\left(\mathbb{E}_{P_r}\left[ X^n \right] - \mathbb{E}_{P_g} \left[ X^n \right]\right).
\end{equation}

\section{Experiments}
In the following subsections, we describe our experimental procedure.
\subsection{Domains}
We evaluated our method against three different synthetic data sources and one real-world data source. Our synthetic data sources consisted of a ``sawtooth'' distribution, a discrete distribution with three possible values, and a mixture of two Gaussian distributions. Each distribution was sampled $10,000$ times to construct a dataset that was then used across all experiments and models.  These distributions correpond to the following random variables defined below:
\begin{equation}
\begin{aligned}
&X_\text{sawtooth} = \sqrt{Y_1} - 1\\
&X_\text{discrete} = \frac{1}{2}\left(-\boldsymbol{1}\{ Y_2 < 0.25\} + \boldsymbol{1}\{Y_2 > 0.75\}\right)\\
&X_\text{mixture} = B N_1 + (1 - B) N_2,
\end{aligned}
\end{equation}
where $Y_1, Y_2 \sim \text{Uniform}(0, 1)$, $B \sim \text{Bernoulli}(0.5)$, $N_1 \sim \text{Normal}(0.5, 0.05)$ and $N_2 \sim \text{Normal}(-0.5, 0.05)$. 

Our real-world data source is a collection of city populations from the Free World Cities Database (https://www.maxmind.com/en/free-world-cities-database). We pre-processed this data by applying a logarithmic scaling to the population numbers and normalizing the resulting log-populations to be between $-1$ and $1$. We denote the random variable associated with this data source as $X_\text{cities}$. 

Figure~\ref{fig:true_histograms} characterizes each of these data sources by sampling a million points from each and plotting their histograms.

\begin{figure}
    \centering
    \includegraphics[width=\textwidth/3]{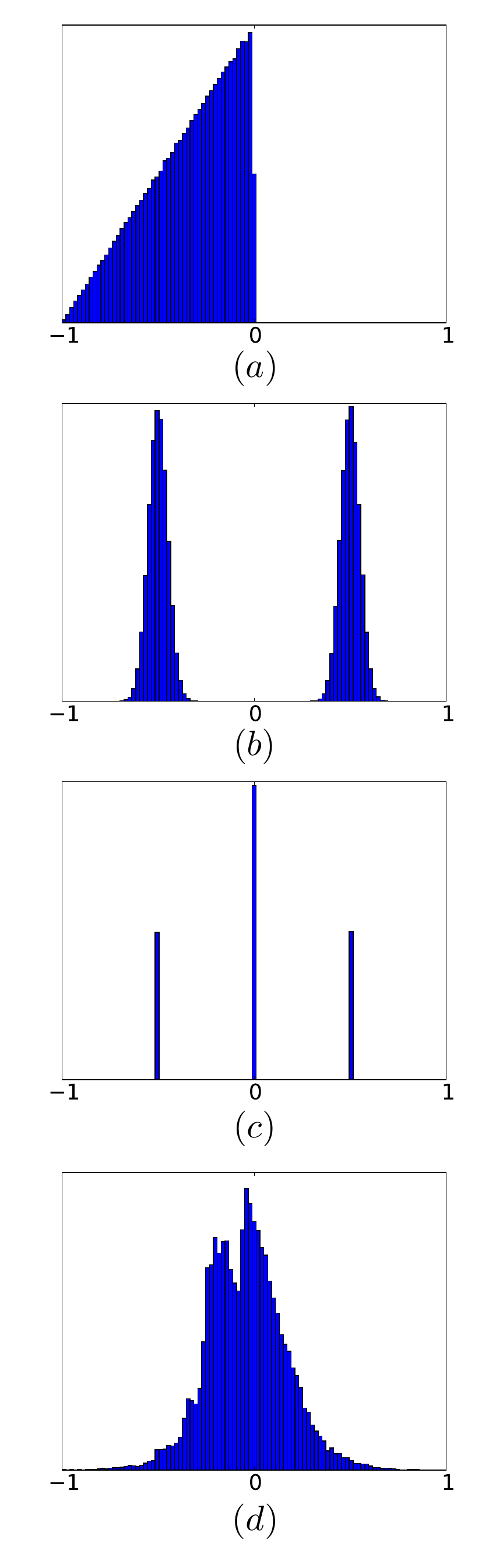}
    \caption{Histograms generated by drawing 1,000,000 samples of random variables (a) $X_\text{sawtooth}$, (b) $X_\text{mixture}$, (c) $X_\text{discrete}$ and (d) $X_\text{cities}$.}
    \label{fig:true_histograms}
\end{figure}

\subsection{Network Architectures}

To maintain consistency between experiments, we used the same generator network architecture for both the wGAN-GP experiments and for our method. This generator network architecture consists of $3$ batch-normalized, fully connected layers with $500$ neurons each and leaky ReLU activation, followed by a single fully connected output layer with $1$ neuron and a $\text{tanh}$ activation. The wGAN-GP experiment used a discriminator with $2$ fully connected layers with $100$ neurons each and leaky ReLU activations, followed by a single fully connected output layer with $1$ neuron and linear activation. Following \namecite{Gulrajani2017ImprovedTO}, we used $\lambda = 10$ to enforce constraints across all experiments. Additionally, we used the AdamOptimizer \cite{Kingma2014AdamAM} with $\beta_1 = 0$, $\beta_2 = 0.9$ and a learning rate of $0.0001$. For all reparameterized critic models we clipped the infinite sums in the expansions at $N = 20$. Additionally, batch normalization \cite{Ioffe2015BatchNA} is used for all generator networks in our experiments. 

\subsection{Evaluation Procedure}

All of the comparison algorithms attempt to learn a representation of the target 1-dimensional probability distribution. For each model, we measure its accuracy by computing the sample Earth-Mover's distance. This quantity is computed by sampling the model $10,000$ times and constructing a histogram out of its sampled outputs. The entries in these histograms are then normalized so that the sum of the bin values is $1$. A similar histogram is then constructed using the true data source, and the Earth-Mover's distance is computed between them. For computing the Earth-Mover's distance, we used the publicly available Python library pyEMD.
For each of the GAN methods, training was conducted over $100,000$ iterations, with an estimate of the Earth-Mover's distance being computed with the training data every $1000$ iterations. At the end of training, the lowest estimate over the course of training is reported as the model's Earth-Mover's distance (EMD). 

\subsection{Results}

We present the results of running $4$ trials for each of the GAN-based models. 
We denote our runs with reparameterized critics as ``Taylor Critic'' and ``Fourier Critic'' for the Fourier Series and Taylor Series reparameterizations, respectively. The best obtained Earth-Mover's distances for each run and model are reported in Tables \ref{tab:gm_runs}, \ref{tab:disc_runs}, \ref{tab:sawtooth_runs} and \ref{tab:cities_runs}. We additionally report the average Earth-Mover's distances across the 4 trials and compare these numbers to the performance of a Kernel Density Estimator as a nonparametric baseline. These results are posted in Table~\ref{tab:average_results}.
\begin{table}[!htb]
\centering
\begin{tabular}{||c||c|c|c|c||} 
\hline
 & $X_\text{mixture}$ & $X_\text{discrete}$ & $X_\text{sawtooth}$ & $X_\text{cities}$ \\ [0.5ex] 
 \hline\hline
 KDE & 0.0073 & 0.01002 & 0.0040 & 0.0027
 \\ 
 \hline
 wGAN-GP & 0.0822 & 0.1318 & 0.26055 & 0.0188 \\
 \hline
 Taylor Critic & 0.0216 & \textbf{0.0106} & 0.0151 & \textbf{0.0096} \\
 \hline
 Fourier Critic & \textbf{0.0186} & 0.0193 & \textbf{0.0109} & 0.0103 \\
 \hline
\end{tabular}
    \caption{Table containing the average Earth-Mover's Distances over the $4$ runs detailed in Tables \ref{tab:gm_runs}, \ref{tab:disc_runs}, \ref{tab:sawtooth_runs} and \ref{tab:cities_runs} for each GAN-based model.}\label{tab:average_results}
\end{table}

\begin{table}[!htb]
\centering
\begin{tabular}{||c||c|c|c|c||} 
\hline
& \multicolumn{4}{c||}{$X_\text{mixture}$} \\
\hline 
 & $1$ & $2$ & $3$ & $4$ \\ [0.5ex] 
 \hline\hline
 wGAN-GP & 0.0206 & 0.0279 & \emph{0.2578} & 0.0226 \\
 \hline
 Taylor Critic & 0.0179 & 0.0216 & 0.0217 & 0.0250 \\
 \hline
 Fourier Critic & 0.0204 & \textbf{0.0164} & 0.0201 & 0.0175 \\
 \hline
\end{tabular}
    \caption{Table of Earth-Mover's distances for $4$ runs of the wGAN-GP, Taylor Critic and Fourier Critic on the Gaussian Mixture dataset. Run~3 of the wGAN-GP illustrates its instability.}\label{tab:gm_runs}
\end{table}

\begin{table}[!htb]
\centering
\begin{tabular}{||c||c|c|c|c||} 
\hline
& \multicolumn{4}{c||}{$X_\text{discrete}$} \\
\hline 
 & $1$ & $2$ & $3$ & $4$ \\ [0.5ex] 
 \hline\hline
 wGAN-GP & 0.1381 & 0.1333 & 0.1314 & 0.1242 \\
 \hline
 Taylor Critic & \textbf{0.0091} & 0.0121 & 0.0103 & 0.0110 \\
 \hline
 Fourier Critic & 0.0129 & 0.0129 & 0.0287 & 0.0226 \\
 \hline
\end{tabular}
    \caption{Table of Earth-Mover's Distances for $4$ runs of the wGAN-GP, Taylor Critic and Fourier Critic on the Discrete dataset.}\label{tab:disc_runs}
\end{table}

\begin{table}[!htb]
\centering
\begin{tabular}{||c||c|c|c|c||} 
\hline
& \multicolumn{4}{c||}{$X_\text{sawtooth}$} \\
\hline 
 & $1$ & $2$ & $3$ & $4$ \\ [0.5ex] 
 \hline\hline
 wGAN-GP & \emph{0.4891} & 0.0226 & \emph{0.4653} & 0.0652 \\
 \hline
 Taylor Critic & 0.0157 & 0.0133 & 0.0152 & 0.0161 \\
 \hline
 Fourier Critic & \textbf{0.0081} & 0.0132 & 0.0132 & 0.0091 \\
 \hline
\end{tabular}
    \caption{Table of Earth-Mover's distances for $4$ runs of the wGAN-GP, Taylor Critic and Fourier Critic on the Sawtooth dataset.  Runs~1 and~3 of the wGAN-GP illustrate its instability.}\label{tab:sawtooth_runs}
\end{table}

\begin{table}[!htb]
\centering
\begin{tabular}{||c||c|c|c|c||} 
\hline
& \multicolumn{4}{c||}{$X_\text{cities}$} \\
\hline 
 & $1$ & $2$ & $3$ & $4$ \\ [0.5ex] 
 \hline\hline
 wGAN-GP & 0.0225 & 0.0189 & 0.0182 & 0.0157 \\
 \hline
 Taylor Critic & 0.0120 & 0.0103 & \textbf{0.0066} & 0.0094 \\
 \hline
 Fourier Critic & 0.0108 & 0.0069 & 0.0086 & 0.0150 \\
 \hline
\end{tabular}
    \caption{Table of Earth-Mover's distances for $4$ runs of the wGAN-GP, Taylor Critic and Fourier Critic on the City Population dataset.}\label{tab:cities_runs}
\end{table}

We observe that both models with reparameterized critics significantly outperform wGAN-GP and are frequently competitive with Kernel Density Estimation. From Tables \ref{tab:gm_runs}, \ref{tab:disc_runs}, \ref{tab:sawtooth_runs} and \ref{tab:cities_runs} we observe that the  reparameterized critic models' worst runs are generally better than the wGAN-GP model's best runs, and the reparameterized critic models have significantly lower variance across runs than wGAN-GP.

Since all GAN-based methods in this paper have the same network architecture for their generators, it is reasonable to attribute this difference to the forms of the critics. As we showed in Theorem \ref{thm:global_max}, the process of optimizing the critic with respect to a given generator cannot ``get stuck'' in some locally maximal region of the space of critics. Thus, as long as the set of critics satisfying $\mathcal{L}_\xi(\cdot) \leq 1$ is sufficiently close to the set of critics satisfying $\| \cdot \|_L \leq 1$, then the generator should always have a clean gradient to follow during its optimization as shown in Lemma~1 of \namecite{Gulrajani2017ImprovedTO}. 

While this does not preclude the possibility that the generator itself could ``get stuck'' during its own optimization against the critic, the difference in consistency across runs between the reparameterized critic models and the wGAN-GP models is evidence that the additional guarantees on reparameterized critics helps empirically. Note that we made every effort to set the hyperparameters of GP-wGAN to reduce or eliminate its instability. It is possible that it would perform better with some other parameter setting, but we were not able to find such a setting. That being said the performance of the reparameterized critic models was relatively unchanged across the parameter settings we explored.

\section{Conclusion and Future Work}
In this work, we illustrated an alternate parameterization of the critic networks that has ideal theoretical properties for gradient-based optimization. We demonstrated that, in the one-dimensional setting, our summable critic models categorically outperform Wasserstein GAN with gradient penalty and are competitive with Kernel Density Estimation on a variety of synthetic and real-world domains. 

While our work on this paper focuses on the one-dimensional setting, there is considerable room to explore extending the approach to higher dimensions. For both the Taylor and Fourier series expansions, there are high-dimensional analogues. These higher-dimensional decompositions generally require exponentially many terms in the number of input dimensions. It may be possible to alleviate this computational cost by exploiting recent techniques to learn sparse polynomials or Fourier series~\cite{Andoni2014LearningSP,Hassanieh2012NearlyOS}. Particularly, while all exponentially many terms of these series may be necessary to model arbitrarily messy functions, it is unlikely that all or even most of them will be required to reasonably approximate the space of $1$-Lipschitz functions.


\newpage

\bibliographystyle{aaai}
\bibliography{mlittman}

\end{document}